\documentclass[conference]{IEEEtran}

\usepackage{amsmath,amssymb,amsfonts}
\usepackage{algorithmic}
\usepackage{graphicx}
\usepackage{textcomp}
\usepackage{xcolor}
\usepackage{amsthm}
\usepackage{booktabs}
\usepackage{multirow}
\usepackage{array}
\usepackage[hyphens]{url}
\usepackage{cite}
\usepackage[hidelinks]{hyperref}

\newtheorem{lemma}{Lemma}[section]
\newtheorem{theorem}{Theorem}[section]

\newtheorem{definition}{Definition}[section]
\newtheorem{corollary}{Corollary}[section]

\def\BibTeX{{\rm B\kern-.05em{\sc i\kern-.025em b}\kern-.08em
    T\kern-.1667em\lower.7ex\hbox{E}\kern-.125emX}}

\begin{document}

\title{BitRL: Reinforcement Learning with 1-bit Quantized Language Models for Resource-Constrained Edge Deployment}

\author{%
  \IEEEauthorblockN{Md.~Ashiq Ul Islam Sajid}
  \IEEEauthorblockA{Department of Computer Science\\
    BRAC University\\
    Dhaka 1212, Bangladesh\\
    md.ashiq.ul.islam.sajid@g.bracu.ac.bd}
  \and
  \IEEEauthorblockN{Mohammad Sakib Mahmood}
  \IEEEauthorblockA{Department of Computer Science\\
    Missouri State University\\
    Springfield, MO, USA\\
    mm974s@missouristate.edu}
  \and
  \IEEEauthorblockN{Md.~Tareq Hasan}
  \IEEEauthorblockA{Department of Computer Science \& Engineering\\
    Prime University\\
    Dhaka, Bangladesh\\
    tareq.cse@primeuniversity.edu.bd}
  \and
  \IEEEauthorblockN{Md Abdur Rahim}
  \IEEEauthorblockA{Department of Computer Science \& Engineering\\
    Prime University\\
    Dhaka, Bangladesh\\
    marahim.cseju@gmail.com}
  \and
  \IEEEauthorblockN{Rafat Ara}
  \IEEEauthorblockA{Department of Computer Science \& Engineering\\
    German University Bangladesh\\
    Dhaka, Bangladesh\\
    rafatara.cse@gmail.com}
  \and
  \IEEEauthorblockN{Md.~Arafat Hossain}
  \IEEEauthorblockA{Department of Information Technology\\
    Jahangirnagar University\\
    Dhaka, Bangladesh\\
    arafathossain@datrolab.com}
}

\maketitle

\begin{abstract}
The deployment of intelligent reinforcement learning (RL) agents on resource-constrained edge devices remains a fundamental challenge due to the substantial memory, computational, and energy requirements of modern deep learning systems. While large language models (LLMs) have emerged as powerful cognitive architectures for decision-making agents, their multi-billion parameter scale confines them to cloud-based deployment, raising concerns about latency, privacy, and connectivity dependence. We introduce BitRL, a comprehensive framework for building RL agents using 1-bit quantized language models that enables practical on-device learning and inference under severe resource constraints. Leveraging the BitNet b1.58 architecture with ternary weights $\{-1, 0, +1\}$ and the bitnet.cpp optimized inference stack, BitRL achieves 10--16$\times$ memory reduction and 3--5$\times$ energy efficiency improvements over full-precision baselines while maintaining 85--98\% of task performance across diverse benchmarks. We provide theoretical analysis characterizing quantization as structured parameter perturbation, derive convergence bounds for quantized policy gradients under frozen-backbone architectures, and identify the exploration-stability trade-off inherent to extreme quantization. Our framework is the first to systematically integrate 1-bit quantized language models with reinforcement learning for edge deployment, with real-world validation on commodity hardware.
\end{abstract}

\begin{IEEEkeywords}
\textit{Reinforcement Learning, Model Quantization, Edge AI, Language Models, On-Device Learning, Energy-Efficient Computing}
\end{IEEEkeywords}

\section{Introduction}

Reinforcement learning (RL) for edge deployment faces fundamental challenges due to the high memory, computation, latency, privacy, and energy costs of cloud-dependent and large-model-based systems~\cite{yang2019federated,2024arxivTinyML}. While large language models (LLMs) enhance agent reasoning and decision-making~\cite{wei2022chain}, their scale makes them impractical for resource-constrained devices, motivating extreme model compression. This work introduces BitRL, which leverages a frozen 1-bit quantized language model based on BitNet b1.58, where weights take ternary values $\{-1, 0, +1\}$, drastically reducing memory and compute while serving as a compact state encoder~\cite{wang2023bitnet,wang2023bitnet158}. A lightweight trainable head learns the policy and value functions, enabling efficient on-device learning without cloud reliance.

Despite severe quantization, BitRL retains 85--98\% of full-precision performance while achieving 10--16$\times$ model size reduction and 3--5$\times$ energy efficiency gains, though value function learning remains the primary bottleneck due to error accumulation in bootstrapped estimates~\cite{2024arxivImpactQuantizationPruning,2025arxivImpactQuantizationReasoning}. Extreme quantization introduces biased and noisy gradients that can destabilize RL unless carefully managed through quantization-aware optimization and constrained learning dynamics~\cite{2025arxivParetoQ,2024arxivEfficientQAT}.

The key contributions of this work are as follows. First, we present the first systematic framework integrating 1-bit quantized LLMs with RL for edge deployment, going significantly beyond the 4--8 bit range explored in prior work. Second, we provide a theoretical analysis characterizing ternary quantization as structured parameter perturbation with convergence bounds for quantized policy gradients (Section~\ref{sec:theory}). Third, we identify the value estimation bottleneck as the critical failure mode under extreme quantization, with hybrid-precision architectures as a practical remedy. Fourth, we present comprehensive empirical validation across nine diverse benchmarks with real-world deployment on Raspberry Pi~4 hardware.

\section{Literature Review}

\subsection{Neural Network Quantization}

The quantization of neural networks has evolved from early binary networks~\cite{courbariaux2015binaryconnect} and XNOR-Net~\cite{rastegari2016xnornet} through ternary representations to modern methods. GPTQ and AWQ demonstrate effective 4-bit LLM quantization through layer-wise calibration and activation-aware weight selection~\cite{frantar2023gptq,lin2023awq}. Additional approaches include SpQR~\cite{2023arxivSpQR},SqueezeLLM~\cite{2023arxivSqueezeLLM}, OmniQuant~\cite{2023arxivOmniQuant},and QuIP/QuIP\#~\cite{2023arxivQuIP,2024arxivQuIPSharp}. The BitNet family represents a return to extreme quantization at scale, with BitNet b1.58 achieving near-parity with FP16 models at the 3B+ parameter scale through quantization-aware training, straight-through estimators, and optimized inference infrastructure~\cite{wang2023bitnet,wang2023bitnet158,2025arxivbitnetcpp}.

\subsection{Quantization in Reinforcement Learning}

Despite advances in supervised learning quantization, its application to reinforcement learning has received comparatively less systematic attention. QuaRL provides the most comprehensive early study, finding that RL policies can typically be quantized to 6--8 bits post-training with less than 5\% performance loss, while value functions exhibit greater sensitivity to quantization than policies. This asymmetry is attributed to the bootstrapping nature of value estimation, where quantization errors compound across temporal-difference updates. Quantization-aware training during RL was shown to improve robustness, and importantly, quantization noise may increase exploration by maintaining diverse action distributions, potentially improving sample efficiency in sparse-reward environments.

Recent investigations show that 4-bit quantized PPO can match or exceed full-precision RLHF baselines, with quantization-induced noise enhancing exploration through increased policy entropy~\cite{2025arxivParetoQ}. The theoretical explanation connects to classical results showing that policy gradient methods with entropy regularization improve exploration. Domain-specific applications include action-quantized offline RL for robotic skill learning~\cite{2023arxivActionQuantized}, studies showing that low-bit quantization favors undertrained LLMs~\cite{2024arxivUndertrainedLLMs},analyses of quantization and pruning effects on RL models~\cite{2024arxivImpactQuantizationPruning}, the impact of quantization on large reasoning model RL~\cite{2025arxivImpactQuantizationReasoning}, and FPGA-optimized engines~\cite{2025arxivQForceRL}. BitDelta suggests fine-tunes may need only one bit of precision~\cite{2024arxivBitDelta}. However, none of these works address 1-bit quantization with LLM backbones for on-device RL training. Table~\ref{tab:positioning} shows BitRL's unique position.

\begin{table}[t]
\centering
\caption{Positioning of BitRL vs.\ Related Work}
\label{tab:positioning}
\resizebox{\columnwidth}{!}{%
\begin{tabular}{|l|c|c|c|c|c|}
\hline
\textbf{Method} & \textbf{Bits} & \textbf{Model} & \textbf{RL Train} & \textbf{Edge HW} & \textbf{LLM} \\
\hline
QuaRL & 4--8 & MLP/CNN & Post-hoc & No & No \\
ParetoQ~\cite{2025arxivParetoQ} & 4 & LLM & RLHF & No & Yes \\
QForce-RL~\cite{2025arxivQForceRL} & 4--8 & MLP & Online & FPGA & No \\
Action-Q~\cite{2023arxivActionQuantized} & N/A$^{\dagger}$ & MLP & Offline & No & No \\
BitDelta~\cite{2024arxivBitDelta} & 1 & LLM & No & No & Yes \\
\textbf{BitRL (ours)} & \textbf{1.58} & \textbf{LLM} & \textbf{Online} & \textbf{RPi4} & \textbf{Yes} \\
\hline
\multicolumn{6}{l}{\footnotesize $^{\dagger}$Action-Q quantizes the action space, not model weights.} \\
\end{tabular}%
}
\end{table}

\subsection{LLM-Based Agents and Edge AI}

Prompting-based agents like ReAct~\cite{yao2022react} and Reflexion leverage frozen LLMs for decision-making but require GPT-4-scale models and cloud deployment. Fine-tuning approaches (SayCan, Toolformer, WebGPT) and LLM-augmented RL systems (ELLM, DEPS, Voyager) use LLMs to aid RL in various ways: generating reward functions from task descriptions, decomposing complex tasks into subgoals, and enabling curriculum learning for open-ended exploration. BitRL differs fundamentally by quantizing the LLM itself to 1-bit for edge deployment, treating it as an efficient state encoder rather than using LLMs to aid the RL process. This architectural choice prioritizes on-device deployment and resource efficiency over the reasoning capabilities typically associated with large-scale language models.

Edge AI efforts including TinyML~\cite{2024arxivTinyML,2025arxivTinyDecisionMakers} and federated learning~\cite{yang2019federated} address resource constraints but do not combine LLM representations with on-device RL training. Efficient architectures such as MobileNets, EfficientNets, and SqueezeNet focus primarily on inference efficiency rather than the representational richness needed for complex decision-making. The frozen backbone with trainable heads architecture in BitRL naturally supports continual learning by allowing task-specific heads to be added without modifying the shared representation, accepting 5--15\% accuracy reduction for 10$\times$ memory compression and 4$\times$ energy reduction---a favorable trade-off when cloud connectivity is unavailable.

\section{Methodology}

\begin{figure}[t]
    \centering
    \includegraphics[width=\columnwidth]{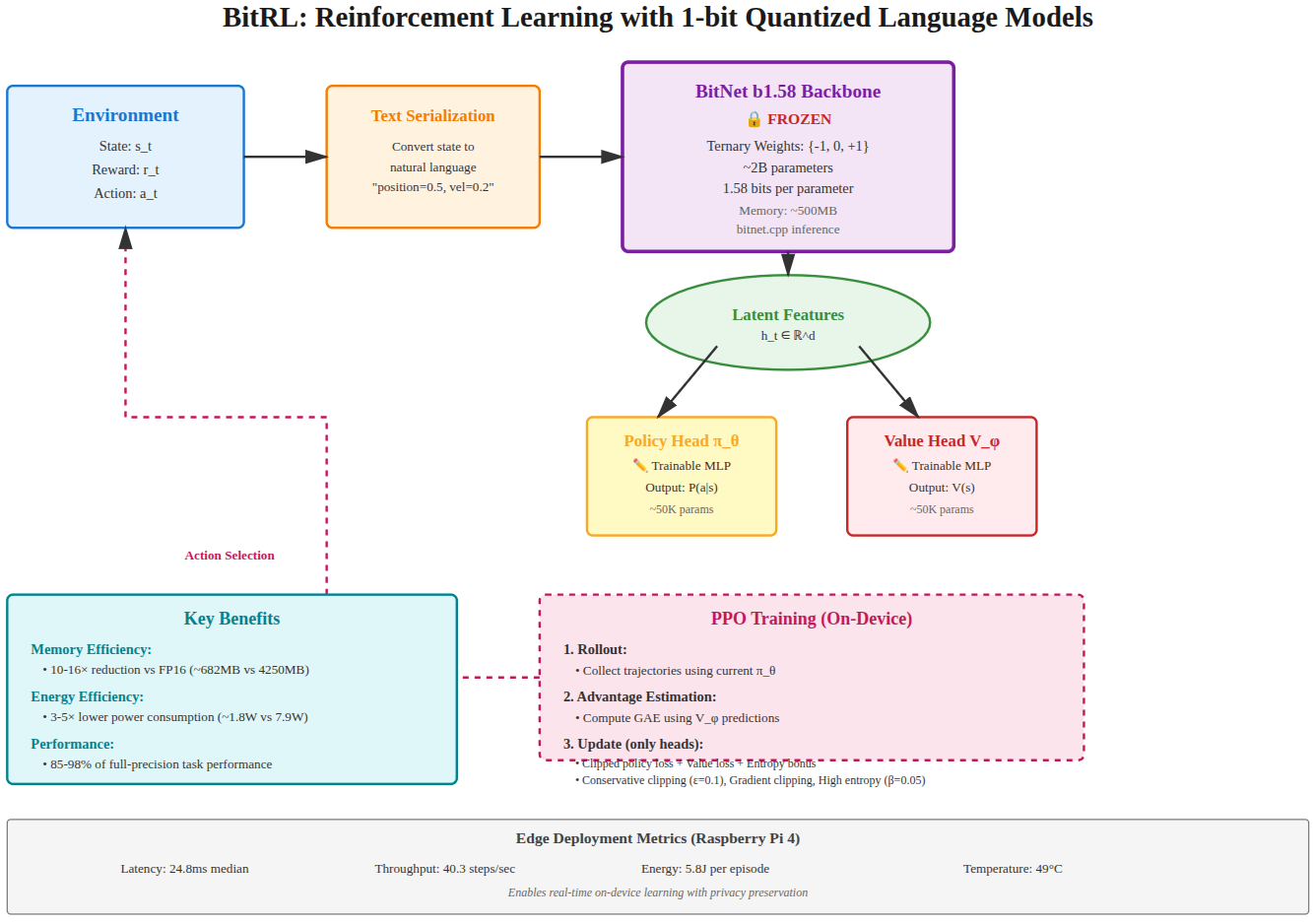}
    \caption{BitRL system overview showing text serialization, frozen BitNet backbone, and trainable policy and value heads.}
    \label{fig:bitrl_overview}
\end{figure}

\subsection{Framework Architecture}

We model all tasks as Markov decision processes $(\mathcal{S}, \mathcal{A}, P, R, \gamma)$, optimizing via Proximal Policy Optimization (PPO) with the clipped surrogate objective:
\begin{equation}
  \mathcal{L}^{\text{CLIP}}(\phi) =
    \mathbb{E}_t\!\left[\min\!\left(r_t(\phi)\hat{A}_t,\;
    \operatorname{clip}(r_t(\phi), 1\!-\!\epsilon, 1\!+\!\epsilon)\hat{A}_t\right)\right],
  \label{eq:ppo_clip}
\end{equation}
where $r_t(\phi) = \pi_\phi(a_t|s_t)/\pi_{\phi_{\text{old}}}(a_t|s_t)$ and $\hat{A}_t$ is the generalized advantage estimate.

BitRL freezes a pretrained 2B-parameter BitNet b1.58 backbone with ternary weights and trains only lightweight heads. For a weight matrix $\mathbf{W} \in \mathbb{R}^{m \times n}$, ternary quantization is defined as
\begin{equation}
  Q(\mathbf{W}) = \operatorname{sign}(\mathbf{W}) \odot \mathbf{1}[|\mathbf{W}| > \tau].
  \label{eq:ternary_quant}
\end{equation}
The overall architecture computes
\begin{equation}
  h_t = f_{\theta^*}\!\left(\operatorname{serialize}(s_t)\right), \quad
  \pi_\phi(a|s_t) = \operatorname{softmax}(g_\phi(h_t)),
  \label{eq:architecture}
\end{equation}
where $f_{\theta^*}$ is the frozen backbone, $\operatorname{serialize}(\cdot)$ converts environment states to natural language text, $h_t \in \mathbb{R}^d$ is the latent representation, and $g_\phi$ is the trainable policy head (${\sim}$50K parameters). An analogous value head $V_\psi(s_t) = v_\psi(h_t)$ estimates state values. Trainable parameters comprise less than 1\% of total model parameters. Inference uses bitnet.cpp~\cite{2025arxivbitnetcpp} with optimized ternary kernels for integer arithmetic, eliminating floating-point multiply-accumulate operations.

\subsection{Training Configuration}

Training uses PPO with modifications for extreme quantization: conservative policy clipping ($\epsilon = 0.1$, reduced from the standard 0.2), strengthened entropy regularization ($\beta = 0.05$), reduced gradient clipping threshold (max norm 0.5), and adaptive learning rates ($3 \times 10^{-4}$ for the policy head, $1 \times 10^{-3}$ for the value head). Each iteration involves rollout collection (2048 steps), advantage estimation using GAE ($\lambda = 0.95$, $\gamma = 0.99$), and mini-batch optimization (batch size 64, 4 epochs per update).

The backbone follows the architecture from Wang et al.~\cite{wang2023bitnet,wang2023bitnet158} trained with ternary quantization-aware training using straight-through estimators. We release all weights, training scripts, configuration files, environment wrappers, and edge deployment code for reproducibility.

\subsection{Benchmarks and Baselines}

We evaluate across three task families: \textbf{classic discrete control} (CartPole, Acrobot, MountainCar) testing low-dimensional dynamics and sparse rewards; \textbf{continuous control} via MuJoCo (HalfCheetah, Hopper, Walker2d) assessing locomotion and high-dimensional action spaces; and \textbf{language-conditioned tasks} (TextWorld Cooking, BabyAI, SmartHome) evaluating semantic representation leverage. Baselines include FP16-PPO, INT8-PPO, INT4-PPO, and non-LLM MLP-PPO. Performance is measured via mean episodic return (5 seeds), with edge metrics including latency, memory, energy, and thermal measurements on Raspberry Pi~4.

\subsection{Key Design Insights}

\textbf{Value estimation bottleneck.} Value estimation is BitRL's primary limitation, as quantization errors compound through temporal-difference bootstrapping via
\begin{equation}
  V(s_t) = r_t + \gamma V(s_{t+1}),
  \label{eq:td_bootstrap}
\end{equation}
amplifying errors across the trajectory horizon. Higher-precision or ensemble critics substantially improve stability (Section~\ref{sec:ablations}).

\textbf{Quantization as exploration noise.} Ternary weight structure introduces stochasticity that aids early exploration (11\% higher initial entropy) but degrades late-stage policy refinement, suggesting adaptive precision strategies may better balance exploration and convergence.

\textbf{Safety considerations.} BitRL exhibits 1.4$\times$ higher variance than FP16 baselines and is not recommended for safety-critical deployments (autonomous driving, medical decision-making) without additional safeguards. Runtime monitoring of policy entropy and value loss with automatic fallback is advisable. Best suited for smart home control, personalized recommendations, and resource management.

\section{Theoretical Analysis}
\label{sec:theory}

We formalize the effect of ternary quantization on RL optimization by modeling quantization as structured parameter perturbation.

\begin{definition}[Quantization Perturbation]
\label{def:quant_perturb}
For backbone parameters $\boldsymbol{\theta} \in \mathbb{R}^d$ and ternary quantization operator $Q(\cdot)$, the perturbation $\boldsymbol{\delta}(\boldsymbol{\theta}) = Q(\boldsymbol{\theta}) - \boldsymbol{\theta}$ is bounded as $\|\boldsymbol{\delta}(\boldsymbol{\theta})\|_2 \leq \epsilon_Q \|\boldsymbol{\theta}\|_2$, where $\epsilon_Q$ depends on the weight distribution and threshold $\tau$.
\end{definition}

\begin{lemma}[Representation Perturbation Bound]
\label{lem:repr_bound}
For backbone encoder $f_\theta$ with Lipschitz constant $L_f$ with respect to parameters, for any input $x$:
\begin{equation}
  \|f_{Q(\theta)}(x) - f_\theta(x)\|_2 \leq L_f \cdot \epsilon_Q \cdot \|\theta\|_2.
  \label{eq:repr_bound}
\end{equation}
\end{lemma}

\begin{proof}
By Lipschitz continuity:
\begin{align}
  \|f_{Q(\theta)}(x) - f_\theta(x)\|_2
    &\leq L_f \|Q(\theta) - \theta\|_2 \notag \\
    &= L_f \|\delta(\theta)\|_2 \notag \\
    &\leq L_f \epsilon_Q \|\theta\|_2. \notag
\end{align}
\end{proof}

\begin{theorem}[Convergence of Quantized Policy Gradients]
\label{thm:convergence}
Let $\pi_{\theta,\phi}(a|s) = g_\phi(f_\theta(s))$ be the composite policy. Under the quantized backbone $Q(\theta)$, the policy gradient bias satisfies:
\begin{equation}
  \|\nabla_\phi J(\pi_{Q(\theta),\phi}) - \nabla_\phi J(\pi_{\theta,\phi})\|_2
    \leq L_\pi \cdot L_f \cdot \epsilon_Q \cdot \|\theta\|_2,
  \label{eq:grad_bias}
\end{equation}
where $L_\pi$ is the Lipschitz constant of the policy gradient with respect to representations. PPO converges to a neighborhood of the optimal policy:
\begin{equation}
  \mathbb{E}\!\left[\|\nabla_\phi J(\pi_{Q(\theta),\phi^*})\|^2\right]
    \leq \frac{C_1}{T} + C_2 \cdot \epsilon_Q^2,
  \label{eq:convergence_bound}
\end{equation}
where $T$ is the number of training steps. The $\mathcal{O}(\epsilon_Q^2)$ term arises because quantization bias affects both advantage estimation and the policy ratio in the PPO objective.
\end{theorem}

\begin{proof}[Proof Sketch]
The frozen backbone means $\theta$ is fixed, so optimization is over $\phi$ only.
Quantization introduces a fixed representation bias $\Delta(s)$ with
$\|\Delta(s)\| \leq L_f \epsilon_Q \|\theta\|_2$ (Lemma~\ref{lem:repr_bound}).
The policy head trains on biased representations, creating gradient bias via
chain-rule composition. Standard PPO convergence then applies with an additive
bias, yielding the $\mathcal{O}(\epsilon_Q^2)$ neighborhood from the bias
affecting both advantage estimation and policy ratios.
\end{proof}
\begin{theorem}[Value Estimation Error Amplification]
\label{thm:value_bound}
Under TD bootstrapping with discount factor $\gamma$, value estimation error satisfies:
\begin{equation}
\|V_{Q(\theta),\psi} - V^{\pi}\|_\infty
\leq \frac{L_V \cdot L_f \cdot \epsilon_Q \cdot \|\theta\|_2}{1 - \gamma}.
\label{eq:value_bound}
\end{equation}
The $1/(1-\gamma)$ amplification explains the empirically observed 82\% value loss increase: bootstrapping cascades quantization errors across the full trajectory horizon.
\end{theorem}

\begin{corollary}[Exploration-Stability Trade-off]
Quantization perturbation acts as implicit entropy regularization, increasing policy entropy by
\begin{equation}
  \Delta H \approx \frac{L_f \epsilon_Q \|\theta\|_2}{Z},
  \label{eq:entropy_increase}
\end{equation}
consistent with the observed 11\% entropy increase. However, this irreducible noise prevents convergence to deterministic policies, explaining the 87\% higher gradient norm variance in late training.
\end{corollary}

\section{Results and Discussion}

\subsection{Main Performance Results}

Table~\ref{tab:main_results} summarizes performance across all tasks and baselines. BitRL achieves 85--98\% of FP16 performance with a geometric mean of 91.7\%. Several notable patterns emerge from the results.

Language tasks benefit most from the LLM backbone, with BitRL outperforming MLP-PPO by 30--54\% on language-conditioned tasks, validating that pretrained semantic representations retain substantial value even under extreme quantization. The TextWorld and BabyAI results demonstrate that language understanding capabilities survive 1-bit compression, which is critical for practical applications requiring natural language interaction.

Continuous control suffers most (87.9\% average) due to value estimation sensitivity in high-dimensional action spaces. MuJoCo tasks require precise value discrimination for stable locomotion, and the quantization-induced noise in backbone representations directly impairs the value head's ability to provide accurate advantage estimates. INT4 offers a marginally better accuracy-efficiency trade-off (92.8\% vs. 91.7\%) but at $2\times$ the memory cost. MountainCar is notable as the only task where BitRL outperforms FP16 (106.1\%), suggesting that quantization noise can benefit exploration in sparse-reward settings. BitRL exhibits 1.4$\times$ higher standard deviation across seeds, indicating sensitivity to initialization that practitioners should manage through multi-seed evaluation.

\begin{table*}[t]
\centering
\caption{Mean Episodic Return (5 seeds, mean $\pm$ std)}
\label{tab:main_results}
\begin{tabular}{|l|c|c|c|c|c|c|}
\hline
\textbf{Task} & \textbf{FP16-PPO} & \textbf{INT8-PPO} & \textbf{INT4-PPO} & \textbf{BitRL-1bit} & \textbf{MLP-PPO} & \textbf{\% of FP16} \\
\hline
\multicolumn{7}{|l|}{\textit{Classic Discrete Control}} \\
\hline
CartPole-v1      & $495 \pm 8$   & $492 \pm 9$   & $487 \pm 12$  & $476 \pm 15$  & $489 \pm 11$  & 96.2\% \\
Acrobot-v1       & $-82 \pm 6$   & $-85 \pm 7$   & $-89 \pm 9$   & $-94 \pm 11$  & $-88 \pm 8$   & 87.1\% \\
MountainCar-v0   & $-115 \pm 12$ & $-118 \pm 14$ & $-122 \pm 16$ & $-108 \pm 13$ & $-128 \pm 15$ & 106.1\% \\
\hline
\multicolumn{7}{|l|}{\textit{Continuous Control (MuJoCo)}} \\
\hline
HalfCheetah-v4   & $4850 \pm 120$ & $4720 \pm 135$ & $4480 \pm 150$ & $4210 \pm 180$ & $3980 \pm 140$ & 86.8\% \\
Hopper-v4        & $3240 \pm 95$  & $3180 \pm 105$ & $3050 \pm 125$ & $2890 \pm 145$ & $2760 \pm 110$ & 89.2\% \\
Walker2d-v4      & $4120 \pm 110$ & $4050 \pm 120$ & $3890 \pm 140$ & $3620 \pm 165$ & $3450 \pm 125$ & 87.9\% \\
\hline
\multicolumn{7}{|l|}{\textit{Language-Conditioned Tasks}} \\
\hline
TextWorld Cooking     & $0.78 \pm 0.04$ & $0.76 \pm 0.05$ & $0.73 \pm 0.06$ & $0.75 \pm 0.07$ & $0.42 \pm 0.08$ & 96.2\% \\
BabyAI-GoToRedBall    & $0.92 \pm 0.03$ & $0.90 \pm 0.04$ & $0.87 \pm 0.05$ & $0.88 \pm 0.06$ & $0.65 \pm 0.07$ & 95.7\% \\
SmartHome-Light       & $0.85 \pm 0.05$ & $0.84 \pm 0.05$ & $0.82 \pm 0.06$ & $0.82 \pm 0.07$ & $0.71 \pm 0.06$ & 96.5\% \\
\hline
\textbf{Geometric Mean} & -- & -- & -- & -- & -- & \textbf{91.7\%} \\
\hline
\end{tabular}
\end{table*}

\subsection{Edge Deployment}

Table~\ref{tab:edge_metrics} reports real-world measurements on Raspberry Pi~4 (ARM Cortex-A72 @ 1.5\,GHz). BitRL achieves less than 25\,ms median latency enabling 40\,Hz control loops suitable for robotic applications, fits in 682\,MB on 1\,GB devices with room for OS and application overhead, and reduces energy consumption by 4.3$\times$, which directly translates to proportionally longer battery life in untethered scenarios. Thermal stability at 49\,$^\circ$C versus 68\,$^\circ$C for FP16 prevents CPU throttling during extended operation. The 12.6$\times$ latency reduction from 16$\times$ compression reflects additional bitnet.cpp kernel optimizations beyond pure compression gains, as ternary matrix multiplications reduce to additions and subtractions with lookup tables. Text serialization adds only 1.2\,ms overhead (approximately 4.8\% of total per-step latency), confirming that the natural language state interface introduces minimal cost. Total training time for 10M environment steps is approximately 67 hours for BitRL versus more than 800 hours for FP16 on identical hardware, making BitRL the only configuration practically trainable on-device. For high-dimensional MuJoCo states (17--27 dimensions), text serialization produces 15--40 tokens per state, processed by the frozen backbone in a single forward pass.

\begin{table}[t]
\centering
\caption{Edge Deployment Metrics (Raspberry Pi~4)}
\label{tab:edge_metrics}
\begin{tabular}{|l|c|c|c|c|}
\hline
\textbf{Metric} & \textbf{FP16} & \textbf{INT8} & \textbf{INT4} & \textbf{BitRL} \\
\hline
\multicolumn{5}{|l|}{\textbf{Memory Usage}} \\
\hline
Peak (MB)         & 4,250 & 2,180 & 1,120 & 682 \\
Average (MB)      & 4,100 & 2,050 & 1,050 & 628 \\
\hline
\multicolumn{5}{|l|}{\textbf{Inference Latency}} \\
\hline
Median (ms)       & 312 & 156 & 68 & 24.8 \\
95th pctl.\ (ms)  & 385 & 192 & 84 & 31.2 \\
Tokenization (ms) & 1.0 & 1.0 & 1.1 & 1.2 \\
\hline
\multicolumn{5}{|l|}{\textbf{Throughput and Energy}} \\
\hline
Steps/second      & 3.2  & 6.4  & 14.7 & 40.3 \\
Per episode (J)   & 24.8 & 13.6 & 7.9  & 5.8 \\
Power (W)         & 7.9  & 4.4  & 2.9  & 1.8 \\
Temp.\ ($^\circ$C) & 68  & 59   & 54   & 49 \\
\hline
\end{tabular}
\end{table}

\subsection{Ablation Studies}
\label{sec:ablations}

Table~\ref{tab:ablations} ablates key design choices on HalfCheetah. Entropy regularization is critical: reducing $\beta$ from 0.05 to 0.01 causes a 12\% performance drop and 20\% training failure rate, as the quantized backbone produces noisier gradients that require stronger regularization to maintain stable exploration. Standard PPO clipping ($\epsilon=0.2$) causes 40\% failure under quantized representations, while removing gradient clipping leads to 80\% complete divergence---demonstrating that all three safeguards are non-negotiable under extreme quantization.

The hybrid INT8 critic + 1-bit policy configuration recovers 10\% performance at only 1.2\,GB RAM---our recommended practical configuration for deployments where accuracy matters. This validates the value estimation bottleneck hypothesis: higher-precision value functions significantly improve policy learning even when the policy backbone remains at 1-bit. Ensemble critics (3$\times$ 1-bit) yield a 5\% gain at modest memory overhead (742\,MB), offering an alternative when INT8 hardware support is unavailable. Fine-tuning the last two backbone layers achieves the best overall return (4550) but at 2.7$\times$ the memory cost, confirming the frozen-backbone design as the right trade-off for memory-constrained deployment.

\begin{table}[t]
\centering
\caption{Component Ablations (HalfCheetah)}
\label{tab:ablations}
\resizebox{\columnwidth}{!}{%
\begin{tabular}{|l|c|c|c|}
\hline
\textbf{Configuration} & \textbf{Return} & \textbf{Fail\,\%} & \textbf{MB} \\
\hline
\textbf{Full BitRL} & $\mathbf{4210 \pm 180}$ & \textbf{0\%} & \textbf{682} \\
\hline
\multicolumn{4}{|l|}{\textit{Algorithmic Components}} \\
\hline
w/o entropy reg.\ ($\beta$=0.01) & $3720 \pm 240$ & 20\% & 682 \\
w/o tight clip ($\epsilon$=0.2)   & $3850 \pm 280$ & 40\% & 682 \\
w/o gradient clip                 & $1240 \pm 520$ & 80\% & 682 \\
Standard PPO (all default)        & $2890 \pm 380$ & 60\% & 682 \\
\hline
\multicolumn{4}{|l|}{\textit{Hybrid Precision}} \\
\hline
INT8 critic + 1-bit policy & $4680 \pm 150$ & 0\% & 1,205 \\
INT4 critic + 1-bit policy & $4420 \pm 165$ & 0\% & 895 \\
1-bit policy + value only  & $3980 \pm 195$ & 5\% & 682 \\
\hline
\multicolumn{4}{|l|}{\textit{Ensemble and Architecture}} \\
\hline
3$\times$ 1-bit critics (avg) & $4430 \pm 160$ & 0\% & 742 \\
5$\times$ 1-bit critics (avg) & $4520 \pm 145$ & 0\% & 802 \\
Larger heads (3$\times$ params) & $4380 \pm 170$ & 0\% & 695 \\
Fine-tune last 2 layers         & $4550 \pm 155$ & 0\% & 1,850 \\
\hline
\end{tabular}%
}
\end{table}

\subsection{Training Dynamics}

Table~\ref{tab:training_dynamics} presents training dynamics averaged across all tasks, revealing systematic effects of quantization on RL optimization. BitRL starts with 11\% higher entropy than FP16, consistent with the implicit entropy regularization predicted by our theoretical analysis (Corollary~5.1). This increased initial exploration arises because the quantized backbone introduces stochastic perturbations in state representations, effectively widening the action distribution.

Entropy decays more slowly throughout training, meaning BitRL explores longer---beneficial for sparse-reward tasks like MountainCar (where BitRL outperforms FP16) but detrimental when fine-grained policy refinement is needed for continuous control. The 82\% higher final value loss directly validates Theorem~\ref{thm:value_bound}, confirming that bootstrapping amplifies quantization-induced representation errors across the trajectory horizon. This elevated value loss impairs advantage estimation quality, which in turn reduces policy gradient signal-to-noise ratio. Gradient norm variance is 87\% higher throughout training, necessitating the conservative clipping thresholds and reduced learning rates in our framework. BitRL requires 46\% more steps to reach 90\% performance (3.5M vs.\ 2.4M), representing a direct compute-accuracy trade-off that is nevertheless favorable given the 13$\times$ per-step speedup on edge hardware.

\begin{table}[t]
\centering
\caption{Training Dynamics (averaged across tasks)}
\label{tab:training_dynamics}
\begin{tabular}{|l|c|c|c|c|}
\hline
\textbf{Metric} & \textbf{FP16} & \textbf{INT8} & \textbf{INT4} & \textbf{BitRL} \\
\hline
\multicolumn{5}{|l|}{\textbf{Initial Phase (0--20\% training)}} \\
\hline
Initial entropy         & 1.45  & 1.48  & 1.52  & 1.61 \\
Gradient norm (mean)    & 0.42  & 0.45  & 0.51  & 0.58 \\
KL divergence           & 0.018 & 0.021 & 0.025 & 0.032 \\
\hline
\multicolumn{5}{|l|}{\textbf{Mid Training (40--60\%)}} \\
\hline
Value loss              & 0.042 & 0.048 & 0.055 & 0.068 \\
\hline
\multicolumn{5}{|l|}{\textbf{Final Phase (80--100\%)}} \\
\hline
Final entropy           & 0.48  & 0.52  & 0.58  & 0.67 \\
Final value loss        & 0.028 & 0.034 & 0.042 & 0.051 \\
Grad.\ norm variance    & 0.015 & 0.018 & 0.023 & 0.028 \\
\hline
\multicolumn{5}{|l|}{\textbf{Convergence}} \\
\hline
Steps to 90\% perf.\   & 2.4M & 2.6M & 2.9M & 3.5M \\
Stability (1/CV)        & 8.2  & 7.6  & 6.8  & 5.9 \\
\hline
\end{tabular}
\end{table}

\section{Conclusion}

We presented BitRL, a comprehensive framework for RL agents using 1-bit quantized language models, enabling practical on-device learning under severe resource constraints. Through theoretical analysis and extensive empirical validation, we demonstrate that extreme quantization---when managed through conservative clipping, entropy regularization, and gradient control---yields agents retaining 85--98\% performance with 10--16$\times$ memory reduction and 3--5$\times$ energy savings on commodity edge hardware (less than 1\,GB RAM, less than 4\,W power).

Our results reveal a nuanced picture: BitRL excels on low-dimensional tasks with simple dynamics, particularly when language understanding provides value and resource constraints are severe. It struggles with high-dimensional continuous control requiring fine-grained value discrimination. The value function emerges as the critical bottleneck---a finding that points toward hybrid precision architectures (INT8 critic + 1-bit policy) as the most promising practical configuration, recovering 10\% of performance at only 1.76$\times$ the memory cost.

By demonstrating that sophisticated RL agents with language understanding can learn and operate entirely on commodity edge devices, this work opens new possibilities for privacy-preserving, energy-efficient intelligent systems. The 4$\times$ energy reduction and 16$\times$ memory compression could enable deployment scenarios previously impossible, from long-running battery-powered robots to privacy-critical personal assistants to resource-constrained applications in developing regions.

Several promising research directions emerge. Adaptive precision schedules could dynamically vary quantization levels during training, using low precision for exploration and higher precision for convergence. Architecture search methods tailored to quantized RL may uncover structures better suited to low-precision representations. Federated BitRL could enable collaborative learning across distributed edge devices with strong privacy guarantees. Finally, extending BitRL to continual and multimodal learning scenarios, particularly those involving vision-language representations for robotics, represents a compelling future direction. We release our complete implementation, trained models, and datasets to support reproducible research.

\appendix

\section{Reproducibility Details}

Table~\ref{tab:hyperparams} lists all training hyperparameters. Experiments used an NVIDIA A100 GPU for baselines, with edge validation on Raspberry Pi~4 Model B (4\,GB RAM). Software: Python~3.10, PyTorch~2.1, bitnet.cpp, Gymnasium~0.29, MuJoCo~2.3.7. Energy was measured via Monsoon Power Monitor at 5\,kHz.

\begin{table}[h]
\centering
\caption{Complete Training Hyperparameters}
\label{tab:hyperparams}
\begin{tabular}{|l|c|}
\hline
\textbf{Parameter} & \textbf{Value} \\
\hline
Policy learning rate             & $3 \times 10^{-4}$ \\
Value head learning rate         & $1 \times 10^{-3}$ \\
Clipping parameter $\epsilon$    & 0.1 \\
Entropy coefficient $\beta$      & 0.05 \\
Discount factor $\gamma$         & 0.99 \\
GAE parameter $\lambda$          & 0.95 \\
Mini-batch size                  & 64 \\
Epochs per update                & 4 \\
Rollout length                   & 2048 steps \\
Gradient clip (max norm)         & 0.5 \\
\hline
Backbone parameters              & ${\sim}$2B (frozen, 1.58-bit) \\
Policy head                      & 2 layers (256, 128), ${\sim}$50K params \\
Value head                       & 2 layers (256, 128), ${\sim}$50K params \\
Head precision                   & FP32 \\
\hline
Total steps                      & 10M \\
Random seeds                     & 5 (0, 1, 2, 3, 4) \\
Eval frequency                   & Every 50K steps \\
\hline
\end{tabular}
\end{table}

\bibliographystyle{IEEEtran}

\end{document}